\documentclass[conference,letterpaper,onecolumn]{IEEEtran}

\usepackage{cite}

\usepackage[cmex10]{amsmath}
\usepackage{graphicx,amsfonts,amssymb,amsthm,bbm}
\usepackage{hyperref}
\hypersetup{colorlinks = true, citecolor = blue} 
\usepackage{easybmat,bigdelim,multirow}
\usepackage{algorithm,algorithmic}
\usepackage{balance}
\usepackage[margin=1in]{geometry}

\usepackage{microtype}
\usepackage{subfigure}
\usepackage{booktabs} 

\def\X{{\mathcal{X}}}
\def\Y{{\mathcal{Y}}}
\def\Z{{\mathcal{Z}}}
\def\W{{\mathcal{W}}}
\def\Simp{{\mathcal{P}}}

\def\E{{\mathbb{E}}}

\def\R{{\mathbb{R}}}

\def\1{{\textbf{1}}}
\def\0{{\textbf{0}}}
\DeclareMathOperator{\tr}{tr}
\DeclareMathOperator{\vectorize}{vec}

\DeclareMathOperator*{\argmax}{arg\,max}
\newcommand{\optspace}{\,\,\,}
\newcommand{\norm}[1]{\left\lVert#1\right\rVert}
\newcommand*\hexbrace[2]{%
  \underset{#2}{\underbrace{\rule{#1}{0pt}}}}
	
\newtheorem{proposition}{Proposition}

\newtheorem{definition}{Definition}

\begin{document}

\title{\vspace{0.25in}Probabilistic Clustering using\\Maximal Matrix Norm Couplings}

\author{\IEEEauthorblockN{David Qiu, Anuran Makur, and Lizhong Zheng}
\IEEEauthorblockA{EECS Department, Massachusetts Institute of Technology\\
Email: $\left\{\text{davidq, a\_makur, lizhong}\right\}$@mit.edu}}

\maketitle

\thispagestyle{plain}
\pagestyle{plain}

\begin{abstract}
In this paper, we present a local information theoretic approach to explicitly learn probabilistic clustering of a discrete random variable. Our formulation yields a convex maximization problem for which it is NP-hard to find the global optimum. In order to algorithmically solve this optimization problem, we propose two relaxations that are solved via gradient ascent and alternating maximization. Experiments on the MSR Sentence Completion Challenge, MovieLens 100K, and Reuters21578 datasets demonstrate that our approach is competitive with existing techniques and worthy of further investigation.
\end{abstract}

\section{Introduction}
\label{Introduction}

Clustering is one of many important techniques in unsupervised learning that finds structure in unlabeled data. One important class of clustering algorithms is metric based, where each row of the data matrix corresponds an item's vector representation in $\R^n$. The most well known example of metric based clustering is $k$-means clustering (or Lloyd-Max algorithm \cite{Lloyd1982, Max1960}).

In this paper, we instead focus on probabilistic clustering, where the data matrix is usually viewed as the joint co-occurrences (or affinities) between two discrete sets, $\X$ and $\Y$, of items and users, respectively. The co-occurrence matrix can be normalized to sum to $1$ to represent a joint probability matrix. Much like \cite{DhillonMallelaModha2003}, we want to maximize the ``cluster-to-item'' mutual information over the set of ``user-to-cluster'' assignment matrices. Our main contributions include relaxing this mutual information optimization into a Frobenius norm optimization over ``DTM'' matrices (to be defined later), relating such matrices to graph Laplacians in spectral graph theory, and proposing an alternating maximization algorithm to approximately solve this matrix optimization. Moreover, unlike spectral methods, we directly learn a transition kernel for soft clustering as opposed to following the usual two-step procedure of learning an embedding and then applying $k$-means clustering.


\subsection{Outline}
\label{Outline}

This paper is organized as follows: Section \ref{Background} defines the divergence transition matrix and derives the relationship between its Frobenius norm and mutual information. Section \ref{Maximal Frobenius Norm Coupling} discusses the Frobenius maximization problem for probabilistic clustering and analyzes its convexity and complexity. Section \ref{Optimization Algorithms} relaxes the optimization problem and presents two algorithms based on gradient ascent and alternating maximization, respectively. Section \ref{Experiments} presents some experimental results that validate our model.

\section{Background}
\label{Background}

\subsection{Notation}
\label{Notation}

We let $\X$, $\Y$, and $\Z$ denote the non-empty, finite alphabet sets corresponding to the random variables $X$, $Y$, and $Z$, respectively. For a set $\X$, we let $\Simp_{\X} \subseteq \R^{|\X|}$ denote the probability simplex of probability mass functions (pmfs) on $\X$, and $\Simp_{\X}^{\circ}$ denote the relative interior of $\Simp_{\X}$. Furthermore, for any two sets $\X$ and $\Y$, we let $\Simp_{\Y|\X} \subseteq \R^{|\Y| \times |\X|}$ denote the set of all column stochastic matrices (channels or transition probability kernels) from $\X$ to $\Y$. For convenience, we perceive joint pmfs of any two random variables as matrices, e.g. $\Simp_{\Y \times \X} \subseteq \R^{|\Y| \times |\X|}$, and for any (marginal) pmf $P_X \in \Simp_{\X}$, we let $[P_X] \in \R^{|\X| \times |\X|}$ denote the diagonal matrix with $P_X$ along the principal diagonal. 

For any $m \times n$ real matrix $A \in \R^{m \times n}$, we let $\sigma_1(A) \geq \sigma_2(A) \geq \dots \geq \sigma_{\min(m,n)}(A)$ denote the ordered \textit{singular values} of $A$, and $\tr(A)$ denote the \textit{trace} of $A$. Furthermore, we will use the notation:
\begin{equation}
\norm{A}_{p} \triangleq \left(\sum_{i = 1}^{\min(m,n)}{\sigma_i(A)^p}\right)^{\! \frac{1}{p}}
\end{equation}
to represent the \textit{Schatten $\ell^p$-norm} of $A$ with $1 \leq p \leq \infty$. Two pertinent specializations of the Schatten $\ell^p$-norm are:
\begin{align}
\label{Eq: Nuclear Norm Definition}
\norm{A}_{*} & \triangleq \norm{A}_{1} = \tr\!\left(\left(A^T A\right)^{\frac{1}{2}}\right) \\
\norm{A}_{F} & \triangleq \norm{A}_{2} = \tr\!\left(A^T A\right)^{\frac{1}{2}}
\end{align}
which denote the \textit{nuclear norm} and \textit{Frobenius norm} of $A$, respectively. (Note that in \eqref{Eq: Nuclear Norm Definition}, $(A^T A)^{1/2}$ is the unique positive semidefinite square root matrix of $A^T A$.) Finally, we will use $A \geq 0$ to imply that $A$ is entry-wise non-negative. 

For any two vectors $x,y \in \R^n$, we let $\sqrt{x}$ denote the entry-wise square root of $x$, $\norm{x}_2$ denote the Euclidean $\ell^2$-norm of $x$, and $x y \in \R^n$ denote the (entry-wise) Hadamard product of $x$ and $y$.

\subsection{Information Theoretic Motivation}
\label{Information Theoretic Motivation}

Suppose we are given training data $(Y_1,X_1),\dots,(Y_n,X_n)$ that is drawn i.i.d. from a joint pmf $P_{Y,X} \in \Simp_{\Y \times \X}$ such that $P_Y \in \Simp_{\Y}^{\circ}$ and $P_X \in \Simp_{\X}^{\circ}$. Our goal is to perform \textit{clustering} on $\Y$ by learning the transition probability kernel $P_{Z|Y} \in \Simp_{\Z|\Y}$, where $\Z$ is the set of cluster labels with $|\Z| \ll |\Y|$, and $P_{Z|Y = y} \in \Simp_{\Z}$ represents a soft assignment of $y \in \Y$. Since our training data is ``unlabeled,'' we assume that $X \rightarrow Y \rightarrow Z$ form a Markov chain to extract information about the clusters from our training data. From hereon, we assume that $P_{Y,X}$ is known as it can be empirically estimated from the data, and $P_Z \in \Simp_{\Z}^{\circ}$ is known from some prior domain knowledge. For example, when clustering readers of political blogs, $\X$ is the set of blogs, $\Y$ is the set of readers, and $P_Z$ can be set using priors on the distribution of liberals and conservatives in the country. 

The following information theoretic problem can be used to perform probabilistic clustering:
\begin{equation}
\label{Eq: Information Theoretic Formulation}
\sup_{P_{Z|Y} \in \Simp_{\Z|\Y}: \, P_{Z|Y}P_Y = P_Z} I(X;Z)
\end{equation}
where $P_{X,Y}$ and $P_Z$ are fixed, $X \rightarrow Y \rightarrow Z$ form a Markov chain, and $I(X;Z)$ denotes the mutual information between $X$ and $Z$ (see \cite[Section 2.3]{CoverThomas2006} for a definition). In the sections that follow, we will refer to $P_{Z|Y}P_Y = P_Z$ as the \textit{constraint on the marginal}. Intuitively, the formulation in \eqref{Eq: Information Theoretic Formulation} finds soft clusters by maximizing $I(X;Z)$ and thereby exploiting the information that $X$ contains about $Y$. Note that $I(X;Z) \leq I(X;Y)$ by the data processing inequality \cite[Section 2.8]{CoverThomas2006}, but $P_{Z|Y} = I_{|\Z|}$ (which denotes the $|\Z| \times |\Z|$ identity matrix) is not a solution because $|\Z| \ll |\Y|$.

It is worth mentioning that the formulation in \eqref{Eq: Information Theoretic Formulation} is related to the \textit{information bottleneck method} developed in \cite{TishbyPereiraBialek1999} (which is useful for lossy source compression and clustering), as well as the \textit{linear information coupling problem} introduced in \cite{HuangZheng2012} (which provides intuition about network information theory problems). 

\subsection{Local Approximations}
\label{Local Approximations}

Since the mutual information objective in the probabilistic clustering formulation in \eqref{Eq: Information Theoretic Formulation} has no inherent operational meaning, we will use local approximations, much like \cite{HuangZheng2012}, to transform \eqref{Eq: Information Theoretic Formulation} into a simpler Frobenius norm maximization problem (which is a non-convex quadratic program as shown in section \ref{Maximal Frobenius Norm Coupling}). To this end, for a fixed reference pmf $P_Z \in \Simp_{\Z}^{\circ}$, we define a \textit{locally perturbed pmf} $Q_Z \in \Simp_{\Z}$ of $P_Z$ as follows:
\begin{equation}
Q_Z = P_Z + \epsilon \sqrt{P_Z} \phi
\end{equation}
where $\phi \in \R^{|\Y|}$ is a \textit{spherical perturbation vector} such that $\phi^T \sqrt{P_Z} = 0$ \cite[Equation (14)]{MakurZheng2018}, and $\epsilon \neq 0$ is a scalar that is small enough to ensure that $Q_{Z} \in \Simp_{\Z}$. For such perturbed pmfs $Q_Z$, we can locally approximate the Kullback-Leibler (KL) divergence between $Q_Z$ and $P_Z$ as a scaled Euclidean $\ell^2$-norm of $\phi$. Indeed, as shown in \cite{HuangZheng2012}, a straightforward calculation using Taylor's theorem yields:
\begin{equation}
\label{Eq: Local Approximation of KL Divergence}
D(Q_Z||P_Z) = \frac{1}{2}\epsilon^2 \norm{\phi}_2^2 + o\!\left(\epsilon^2\right)
\end{equation}
where $D(\cdot||\cdot)$ denotes KL divergence (see \cite[Section 2.3]{CoverThomas2006} for a definition), and $o(\epsilon^2)$ represents a function satisfying $\lim_{\epsilon \rightarrow 0}{o(\epsilon^2)/\epsilon^2} = 0$. 

Now consider the following local perturbation relations that we will use to locally approximate \eqref{Eq: Information Theoretic Formulation}:
\begin{equation}
\label{Eq: Soft Clustering Local Approximation}
\forall y \in \Y, \enspace P_{Z|Y = y} = P_Z + \epsilon \sqrt{P_Z} \phi_y 
\end{equation}
where $\{\phi_y \in \R^{|\Z|} : y \in \Y, \, \phi_y^T \sqrt{P_Z} = 0, \, \norm{\phi_y}_2 = 1\}$ are unit norm spherical perturbation vectors, and $\epsilon \neq 0$ is small enough to ensure that $P_{Z|Y = y} \in \Simp_{\Z}$ for every $y \in \Y$. 
Due to the Markov relation $X \rightarrow Y \rightarrow Z$, the conditions in \eqref{Eq: Soft Clustering Local Approximation} imply after some straightforward computation that:
\begin{equation}
\label{Eq: Local Approximation Implication}
\forall x \in \X, \enspace P_{Z|X = x} = P_Z + \epsilon \sqrt{P_Z} \psi_x
\end{equation}
where the spherical perturbation vectors $\{\psi_x \in \R^{|\Z|} : x \in \X, \, \psi_x^T \sqrt{P_Z} = 0\}$ are given by:
\begin{equation}
\forall x \in \X, \forall z \in \Z, \enspace \psi_x(z) = \sum_{y \in \Y}{P_{Y|X}(y|x) \phi_y(z)} .
\end{equation}
To succinctly describe the local approximation of the objective function of \eqref{Eq: Information Theoretic Formulation} that stems from \eqref{Eq: Local Approximation Implication}, we introduce the so called \textit{divergence transition matrices}.

\begin{definition}[Divergence Transition Matrix \cite{HuangZheng2012}]
\label{Def: DTM}
Given a joint pmf $P_{Y,X} \in \Simp_{\Y \times \X}$, with conditional pmfs $P_{Y|X} \in \Simp_{\Y|\X}$ and marginal pmfs satisfying $P_X \in \Simp_{\X}^{\circ}$ and $P_Y \in \Simp_{\Y}^{\circ}$, the divergence transition matrix (DTM) of $P_{Y,X}$ is defined as:
\begin{align}
B_{Y,X} = B(P_{Y,X}) & \triangleq [P_Y]^{-\frac{1}{2}} P_{Y,X} [P_X]^{-\frac{1}{2}} \label{Eq: DTM Definition} \\
& = [P_Y]^{-\frac{1}{2}} P_{Y|X} [P_X]^{\frac{1}{2}} \label{Eq: DTM Definition Conditional}. 
\end{align}
\end{definition}

It is well-known that the largest singular value of $B_{Y,X}$ is $\sigma_1(B_{Y,X}) = 1$ with corresponding right and left singular vectors $\sqrt{P_X}$ and $\sqrt{P_Y}$, respectively (see e.g. \cite{HuangZheng2012}, \cite[Appendix A]{MakurZheng2018}):
\begin{equation}
\label{Eq: DTM Largest Singular Value}
\begin{aligned}
B_{Y,X} \sqrt{P_X} & = \sigma_1(B_{Y,X}) \sqrt{P_Y} = 1 \sqrt{P_Y} , \\
B_{Y,X}^T \sqrt{P_Y} & = \sigma_1(B_{Y,X}) \sqrt{P_X} = 1 \sqrt{P_X} .
\end{aligned}
\end{equation}
Moreover, the next proposition decomposes the DTM of random variables in a Markov chain.

\begin{proposition}[Composed DTM]
	\label{Prop: Composed DTM}
	If $X \rightarrow Y \rightarrow Z$ form a Markov chain, then $B_{Z,X} = B_{Z,Y}B_{Y,X}$.
\end{proposition}
\begin{proof}
	Observe using Definition \ref{Def: DTM} that:
	\begin{align*}
	B_{Z,X} &= [P_Z]^{-\frac{1}{2}}P_{Z|X} [P_X]^{\frac{1}{2}} \\
	&= [P_Z]^{-\frac{1}{2}}P_{Z|Y}P_{Y|X}[P_X]^{\frac{1}{2}} \\
	&= \underbrace{[P_Z]^{-\frac{1}{2}}P_{Z|Y}[P_Y]^{\frac{1}{2}}}_{B_{Z,Y}} \underbrace{[P_Y]^{-\frac{1}{2}}P_{Y|X}[P_X]^{\frac{1}{2}}}_{B_{Y,X}}
	\end{align*}
	where the second equality uses the Markov property.
\end{proof}

Finally, we locally approximate $I(X;Z)$ using \eqref{Eq: Local Approximation Implication}.

\begin{proposition}[Local Approximation of Mutual Information]
\label{Prop: Local Mutual Information}
Under the local perturbation conditions in \eqref{Eq: Local Approximation Implication}, we have:
	\[ I(X;Z) = \frac{1}{2} \left(\norm{B_{Z,X}}_{F}^{2} - 1\right) + o\!\left(\epsilon^2\right) . \]
\end{proposition}

\begin{proof}
Observe that:
\begin{align*}
&I(X;Z) = \sum_{x \in \X}{P_X(x) D(P_{Z|X = x} || P_Z)} \\
	&= \frac{1}{2} \epsilon^2 \sum_{x \in \X}{P_X(x) \norm{\psi_x}_2^2} + o\!\left(\epsilon^2\right) \\
	&= \frac{1}{2} \epsilon^2 \sum_{x,z}{P_X(x) \! \left(\frac{P_{Z|X}(z|x) - P_Z(z)}{\epsilon \sqrt{P_Z(z)}}\right)^{\!\! 2}} + o\!\left(\epsilon^2\right) \\
	&= \frac{1}{2} \sum_{x,z}{\left(\frac{P_{Z,X}(z,x) - P_Z(z)P_X(x)}{\sqrt{P_Z(z) P_X(x)}}\right)^{\!\! 2}} + o\!\left(\epsilon^2\right) \\
	&= \frac{1}{2} \norm{B_{Z,X} - \sqrt{P_Z} \sqrt{P_X}^T}_{F}^{2} + o\!\left(\epsilon^2\right) \\
	&= \frac{1}{2} \left(\norm{B_{Z,X}}_{F}^{2} - 1\right) + o\!\left(\epsilon^2\right)
	\end{align*}
where the first equality follows from a straightforward calculation, the second equality follows from \eqref{Eq: Local Approximation Implication} and \eqref{Eq: Local Approximation of KL Divergence}, the fifth equality follows from Definition \ref{Def: DTM}, and the final equality holds due to \eqref{Eq: DTM Largest Singular Value}.
\end{proof}

We will present the Frobenius norm maximization formulation that follows from applying this local approximation result to \eqref{Eq: Information Theoretic Formulation} in section \ref{Maximal Frobenius Norm Coupling}.

\subsection{Connections to Spectral Graph Theory}
\label{Connections to Spectral Graph Theory}
In the case of $\X = \Y$, if we view $P_{Y|X}$ as a matrix of Markov transition probabilities, \eqref{Eq: DTM Definition Conditional} is the matrix being factorized in diffusion maps \cite{CoifmanLafon2006}. If we view $P_{Y,X}$ as a weighted adjacency matrix, \eqref{Eq: DTM Definition} is almost identical to the symmetric normalized graph Laplacian \cite{Chung1997,BelkinNiyogi2001}. Similar to the Laplacian, the DTM carries an important property that we will use later.

\begin{proposition}
\label{Prop: DTM SV}
The multiplicity of the singular value at $1$ of $B(P_{Y,X})$ is equivalent to the number of connected components in a bipartite graph that has weighted adjacency matrix $P_{Y,X}$.
\end{proposition}

For a proof, we refer readers to \cite[Theorem 3.1.1]{Qiu2017}, which  relates the eigenvalues of the identity minus the Laplacian to the singular values of the (corresponding) DTM.

\section{Maximal Frobenius Norm Coupling}
\label{Maximal Frobenius Norm Coupling}

Inspired by Proposition \ref{Prop: Local Mutual Information}, we will learn the $P_{Z|Y} \in \Simp_{\Z|\Y}$ that probabilistically clusters each $y \in \Y$ by maximizing $\norm{B_{Z,X}}_{F}^2$ instead of $I(X;Z)$. This \textit{Frobenius norm formulation of probabilistic clustering} is presented in the next definition.

\begin{definition}[Frobenius Norm Formulation]
	\label{Def: Frobenius Problem}
	Given a joint pmf $P_{Y,X} \in \Simp_{\Y \times \X}$ so that the marginal pmfs satisfy $P_X \in \Simp_{\X}^{\circ}$ and $P_Y \in \Simp_{\Y}^{\circ}$, and a target pmf $P_Z \in \Simp_{\Z}^{\circ}$, we seek to solve the following extremal problem:
	\begin{equation}
	\label{Eq: Frobenius Problem}
	\max_{P_{Z|Y} \in \Simp_{\Z|\Y}: \, P_{Z|Y}P_Y = P_Z} \norm{B_{Z,X}}_{F}^2
	\end{equation}
where $X \rightarrow Y \rightarrow Z$ form a Markov chain. We will refer to an optimal argument $P_{Z|Y}^{\star}$ of this problem, which represents a desirable soft clustering assignment, as a \emph{maximal Frobenius norm coupling}.
\end{definition}

We make some pertinent remarks about Definition \ref{Def: Frobenius Problem}. Firstly, a ``coupling'' of two marginal pmfs $P_Y$ and $P_Z$ is generally defined as a joint pmf $P_{Z,Y}$ that is consistent with these marginals (and often has additional desirable properties)--see e.g. \cite[Section 4.2]{LevinPeresWilmer2009}. However, since the maximizing conditional pmf $P_{Z|Y}^{\star}$ implicitly defines a joint pmf $P_{Z,Y}^{\star} = P_{Z|Y}^{\star} [P_Y]$, we refer to $P_{Z|Y}^{\star}$ itself as a coupling. Secondly, although the Frobenius norm formulation in \eqref{Eq: Frobenius Problem} can be perceived as a local approximation of \eqref{Eq: Information Theoretic Formulation} (which nicely connects the two problems), we will not actually require $P_{Z|Y}^{\star}$ to be close to $P_Z$ as in \eqref{Eq: Soft Clustering Local Approximation} (i.e. weak dependence between $Z$ and $Y$) when using this formulation. Thirdly, the formulation in \eqref{Eq: Frobenius Problem} is intuitively well-founded because \cite{Huangetal2017} and \cite{Makuretal2015} illustrate that the singular values of the DTM $B_{Z,X}$ capture how informative or correlated mutually orthogonal embeddings of $Z$ and $X$ are. Hence, maximizing the sum of all squared singular values maximizes the relevant dependencies between $Z$ and $X$. Naturally, there are various other reasonable formulations of probabilistic clustering that use singular values of the DTM. We present one such class of formulations in \eqref{Eq: General Problem} in the next subsection.   

\subsection{Theoretical Discussion}
\label{Theoretical Discussion}

Consider the following generalization of \eqref{Eq: Frobenius Problem} that also intuitively captures some notion of probabilistic clustering:
\begin{equation}
\label{Eq: General Problem}
\max_{P_{Z|Y} \in \Simp_{\Z|\Y}: \, P_{Z|Y}P_Y = P_Z} \norm{B_{Z,X}}_{p}^p
\end{equation}
where $P_Z \in \Simp_{\Z}^{\circ}$ and $P_{Y,X} \in \Simp_{\Y \times \X}$ are fixed such that $P_X \in \Simp_{\X}^{\circ}$ and $P_Y \in \Simp_{\Y}^{\circ}$. Using Proposition \ref{Prop: Composed DTM}, we may rewrite the objective function of \eqref{Eq: General Problem} as $\norm{B_{Z,X}}_{p}^p = \|[P_Z]^{-\frac{1}{2}}P_{Z|Y}[P_Y]^{\frac{1}{2}} B_{Y,X}\|_{p}^p$. Since the quantity inside the norm is linear in $P_{Z|Y}$, and the $p$th power of a Schatten $\ell^p$-norm is convex, the objective function is convex. Moreover, the constraints on $P_{Z|Y}$ in \eqref{Eq: General Problem} define a compact and convex set in $\R^{|\Z| \times |\Y|}$. (As a result, the maximum in \eqref{Eq: General Problem} can indeed be achieved due to the extreme value theorem.) Hence, \eqref{Eq: General Problem} is a \textit{maximization of a convex function over a convex set}. While convex functions can be easily minimized over convex sets, non-convex problems like \eqref{Eq: General Problem} are often computationally hard (see e.g. \cite{Bodlaenderetal1990}). 

To illustrate this, we consider the notable special case of \eqref{Eq: General Problem} with $p = 2$ which yields the problem in \eqref{Eq: Frobenius Problem}: 
\begin{equation} 
\label{Eq: Frobenius Problem Alternative Form}
\begin{aligned}
\max_{P_{Z|Y} \in \R^{|\Z| \times |\Y|}} \optspace & \norm{B_{Z,Y} B_{Y,X}}_{F}^2 \\
\text{subject to (s.t.)} \optspace & P_{Z|Y}P_Y = P_Z, \, \1_{|\Z|}^T  P_{Z|Y} = \1_{|\Y|}^T,  \\
& P_{Z|Y} \geq 0
\end{aligned}
\end{equation}
where $\1_k \triangleq [1 \cdots 1]^T \in \R^{k}$, the second and third constraints ensure that $P_{Z|Y} \in \Simp_{\Z|\Y}$, and we use Proposition \ref{Prop: Composed DTM} to rewrite the objective function. Letting $A = B_{Z,Y}$ and $B = B_{Y,X}$, we can straightforwardly rewrite this problem as follows:
\begin{equation}
\label{Eq: Frobenius Problem DTM Form}
\begin{aligned}
\max_{A \in \R^{|\Z| \times |\Y|}} \optspace & \norm{A B}_{F}^2 \\
\text{s.t.} \optspace & A \sqrt{P_Y} = \sqrt{P_Z} , \, A^T \sqrt{P_Z} = \sqrt{P_Y} , \\
& A \geq 0 .
\end{aligned}
\end{equation}
This is clearly a \textit{non-convex quadratic program} (QP). Indeed, letting $a = \vectorize(A) \in \R^{|\Z||\Y|}$ (which stacks the columns of $A$ to form a vector), $M_1 = (B \otimes I_{|\Z|}) (B^T \otimes I_{|\Z|})$, $M_2 = \sqrt{P_Y}^T \otimes I_{|\Z|}$, and $M_3 = I_{|\Y|} \otimes \sqrt{P_Z}^T$, the preceding problem is equivalent to:
\begin{equation}
\label{Eq: Non-convex QP}
\begin{aligned}
\max_{a \in \R^{|\Z||\Y|}} \optspace & a^T M_1 a \\
\text{s.t.} \optspace & M_2 \, a = \sqrt{P_Z}, \, M_3 \, a = \sqrt{P_Y}, \, a \geq 0
\end{aligned}
\end{equation}
where $\otimes$ denotes the Kronecker product, and we use the fact that $\vectorize(ABC) = (C^T \otimes A)\vectorize(B)$ for any matrices $A$, $B$, and $C$ with valid dimensions. The QP in \eqref{Eq: Non-convex QP} is non-convex because $M_1$ is positive semidefinite and we are maximizing the associated convex quadratic form. It is proved in \cite{Sahni1974} that such QPs are \textit{NP-hard} (also see \cite{PardalosVavasis1991, Gao2004} and the references therein). Therefore, there are no known efficient algorithms to exactly solve \eqref{Eq: Frobenius Problem}, and we will resort to relaxations and other heuristics in the ensuing sections.

Finally, we provide some brief intuition for the NP-hardness of \eqref{Eq: Frobenius Problem Alternative Form}. The feasible set of \eqref{Eq: Frobenius Problem Alternative Form} is the convex polytope $\Simp_{\Z|\Y} \cap \mathcal{H}$, where $\mathcal{H} \triangleq \{M \in \R^{|\Z| \times |\Y|}: M P_Y = P_Z\}$ is a $|\Z|(|\Y|-1)$-dimensional affine subspace of $\R^{|\Z| \times |\Y|}$. In general, this convex polytope has super-exponentially many extreme points. To see this, consider the special case where $m = |\Y| = |\Z|$ and $P_Y = P_Z$ are the uniform pmf. Then, $\Simp_{\Z|\Y} \cap \mathcal{H}$ is the set of all $m \times m$ doubly stochastic matrices, and its extreme points are the $m!$ different $m \times m$ permutation matrices by the Birkhoff-von Neumann theorem \cite[Theorem 8.7.2]{HornJohnson2013}. For general $|\Y|$, $|\Z|$, $P_Y$, and $P_Z$, the extreme points of $\Simp_{\Z|\Y} \cap \mathcal{H}$ have more complex structure (see e.g. \cite{JurkatRyser1967}, which studies the uniform $P_Y$ and arbitrary $P_Z$ case). When we maximize a convex function over $\Simp_{\Z|\Y} \cap \mathcal{H}$ as in \eqref{Eq: Frobenius Problem Alternative Form}, the optimum is achieved at an extreme point of $\Simp_{\Z|\Y} \cap \mathcal{H}$. So, we have to search over all super-exponentially many extreme points to find this optimal point. This is computationally very inefficient. 

\subsection{Comparison to Formulations that Directly Modify Co-occurrences}

A key feature of our formulation in \eqref{Eq: Frobenius Problem} is that it clusters using a transition kernel $P_{Z|Y}$ and keeps the original data distribution $P_{Y,X}$ intact. For comparison, lets consider a different optimization problem that clusters by modifying the non-negative co-occurrence matrix $P \in \R^{|\Y| \times |\X|}$ directly:
\begin{equation}
\label{Eq: Community}
\min_{\substack{Q \in \R^{|\Y| \times |\X|}: \\ Q \geq 0}}{\norm{Q-P}_F^2-\lambda \sum_{i=1}^{|\Z|} \sigma_i(B(Q))}
\end{equation}
where $\lambda > 0$ is a hyperparameter that should be set high enough to emphasize the second term in the objective function, $B(Q)$ denotes the DTM corresponding to the joint pmf obtained after normalizing $Q$, and $|\Z|$ represents the ideal number of clusters we want (note that the set $\Z$ is inconsequential in this formulation). Because \eqref{Eq: Community} does not learn a transition kernel, in this subsection we do not normalize the data $P$ to be a valid pmf in order to simplify the presentation.

Intuitively, \eqref{Eq: Community} tries to find the closest non-negative matrix $Q$ that has the top $|\Z|$ singular values as $1$ (i.e. has $|\Z|$ connected components--see Proposition \ref{Prop: DTM SV}). This is closely related to the model in \cite{Nieetal2017} and one drawback of this kind of formulation is that it has $|\Y| |\X|$ parameters to learn. Since the number of clusters is typically much smaller than the number of items, i.e. $|\Z| \ll |\Y|$, our formulation in \eqref{Eq: General Problem} has a much lower number $|\Z| |\Y|$ of parameters to learn.

A more important drawback of \eqref{Eq: Community} is that sometimes, the intuitively correct clustering is not the globally optimal solution. We demonstrate this phenomenon via an example. Let $\X = \X_1 \cup \X_2$ for disjoint sets $\X_1$ and $\X_2$, $\Y = \Y_1 \cup \Y_2$ for disjoint sets $\Y_1$ and $\Y_2$, $|\X_1| = |\X_2| = n$, $|\Y_1| = |\Y_2| = m$, and the number of clusters $|\Z| = 2$. Furthermore, let the data matrix $P$ have the following structure: 
\begin{equation}
P= 
\begin{array}{c@{}c}
\left[
\begin{BMAT}[8pt]{c:c}{c:c}
  s\1 & \phantom{\0} \1 \phantom{\0} \\
  \phantom{\0} \1 \phantom{\0} & s\1
\end{BMAT} 
\right] 
& 
\begin{array}{l}
  \\[-4mm] \rdelim\}{2}{0mm}[$m$] \\ \\ \\[-4mm]  \rdelim\}{2}{0mm}[$m$] \\ \\
\end{array} \\[-1.5ex]
\hexbrace{1.0cm}{n}\hexbrace{1.0cm}{n}
\end{array}
\end{equation}
where $\1$ is a matrix of all $1$'s of appropriate dimension, and $s > 1$ is some scale factor. Clearly, there are two distinct communities, and the intuitive result with two clusters is:
\begin{equation}
\label{Eq: Intuitive Cluster}
Q_1= 
\begin{array}{c@{}c}
\left[
\begin{BMAT}[8pt]{c:c}{c:c}
  s\1 & \phantom{\1} \0 \phantom{\1} \\
  \phantom{\1} \0 \phantom{\1} & s\1
\end{BMAT} 
\right] 
& 
\begin{array}{l}
  \\[-4mm] \rdelim\}{2}{0mm}[$m$] \\ \\ \\[-4mm]  \rdelim\}{2}{0mm}[$m$] \\ \\
\end{array} \\[-1.5ex]
\hexbrace{1.0cm}{n}\hexbrace{1.0cm}{n}
\end{array} 
\end{equation}
where $\0$ is a matrix of all $0$'s of appropriate dimension. Since $Q_1$'s structure creates two connected components, $\X_1 \cup \Y_1$ and $\X_2 \cup \Y_2$, the largest two singular values of $B(Q_1)$ are both $1$. Moreover, the objective function has value $2mn-2\lambda$.

\begin{figure}
\centering
\includegraphics[trim = 0mm 0mm 0mm 0mm, width=0.9\linewidth]{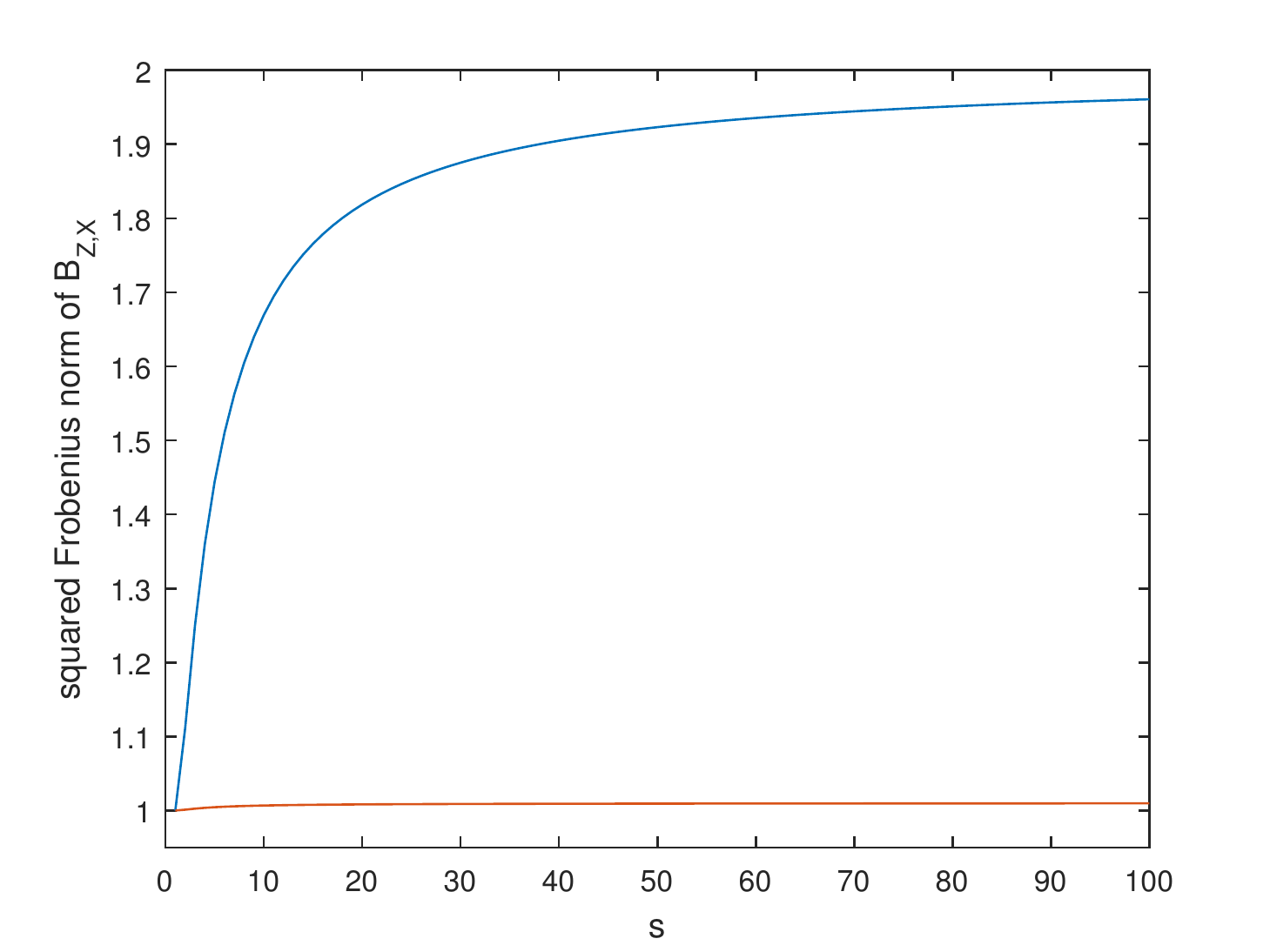}
\caption{Plots of $\norm{B_{Z,X}}_{F}^2$ versus $s \geq 1$ for different transition kernels $P_{Z|Y}$. In particular, the blue plot corresponds to $B_{Z,X}$ defined by $P_{Z|Y}^1$ (the intuitive clustering), and the red plot corresponds to $B_{Z,X}$ defined by $P_{Z|Y}^2$ (the ``one item'' clustering).}
\label{Figure: Frobenius Norm Plot}
\end{figure}

Now consider a different $Q$ that also creates two connected components by only disconnecting one item from $\X$ and one item from $\Y$ from the rest of the items:
\begin{equation} 
\label{Eq: Weird Cluster}
Q_2 = 
\begin{array}{c@{}c}
\left[
\begin{BMAT}[8pt]{ccc:ccc:c}{ccc:ccc:c}
  & & & & & & \\
  & s\1 & & & \1 & & \0 \\
  & & & & & & \\
  & & & & & & \\
  & \1 & & & s\1 & & \0 \\
  & & & & & & \\
  & \0 & & & \0 & & s
\end{BMAT} 
\right] 
& 
\begin{array}{l}
  \\[-2mm] \rdelim\}{4}{5mm}[$m$] \\ \\ \\[6mm]  \rdelim\}{4}{5mm}[$m-1$] \\ \\
  \\[5mm]  \rdelim\}{2}{5mm}[$1$] \\ \\
\end{array} \\[-1ex]
\hexbrace{2.1cm}{n}\hexbrace{2.1cm}{n-1}\hexbrace{0.8cm}{1}
\end{array}
\end{equation}
where the disconnected item forms the bottom $1 \times 1$ block. The largest two singular values of $B(Q_2)$ are still $1$ because of the two connected components. However, the objective function now equals $m + n + s^2 (m+n-2) - 2\lambda$. Thus, when $s < \sqrt{(2mn - m - n)/(m + n - 2)}$, the intuitively correct answer $Q_1$ is not the global optimum of \eqref{Eq: Community}.

In contrast, our maximum Frobenius norm formulation in \eqref{Eq: Frobenius Problem} (without the constraint on the marginal) easily obtains the two intuitive clusters encoded in $P$. For example, let $m = n = 50$, $\Z = \{0,1\}$ denote the cluster labels, and consider the transition kernels $P_{Z|Y}^1$ corresponding to the intuitive clustering shown in \eqref{Eq: Intuitive Cluster} (defined by $P^1_{Z|Y}(0|y) = 1$ for $y \in \Y_1$ and $P^1_{Z|Y}(1|y) = 1$ for $y \in \Y_2$), and $P_{Z|Y}^2$ corresponding to the clustering shown in \eqref{Eq: Weird Cluster} (defined by $P^2_{Z|Y}(0|y) = 1$ for $y \neq y_0$ and $P^2_{Z|Y}(1|y_0) = 1$ for some $y_0 \in \Y_2$). Then, the plots in Figure \ref{Figure: Frobenius Norm Plot} illustrate that the intuitive clustering of $P_{Z|Y}^1$ is greatly preferred by the maximum Frobenius norm formulation. Therefore, our formulation does not exhibit the drawbacks of formulations like \eqref{Eq: Community}.

\section{Optimization Algorithms}
\label{Optimization Algorithms}

To solve the non-convex QP given by the Frobenius norm formulation of probabilistic clustering in \eqref{Eq: Frobenius Problem}, we will use a heuristic gradient ascent algorithm (subsection \ref{Heuristic Gradient Ascent Algorithm}) as well as a nuclear norm relaxation (subsection \ref{Nuclear Norm Relaxation}). Although one approach to finding approximate solutions to an NP-hard problem like \eqref{Eq: Frobenius Problem} is via semidefinite programming (SDP) relaxations, we do not explore SDP based algorithms in this paper. Moreover, many of the simpler SDP relaxations for non-convex QPs do not accurately capture our setting because they only appear to be tight when at least one of the constraints is also quadratic \cite{BaoSahinidisTawarmalani2011}.

\subsection{Heuristic Gradient Ascent Algorithm}
\label{Heuristic Gradient Ascent Algorithm}


We now present a gradient-based algorithm for approximating the maximal Frobenius norm coupling defined by the formulation of probabilistic clustering in \eqref{Eq: Frobenius Problem}, or equivalently, in \eqref{Eq: Frobenius Problem DTM Form}. For computational efficiency, we move the first constraint in \eqref{Eq: Frobenius Problem DTM Form} to the objective function to obtain: 
\begin{equation}
\label{Eq: Relaxed Optimization}
\begin{aligned}
\max_{A \in \R^{|\Z| \times |\Y|}} \optspace & \norm{A B}_{F}^2-\lambda \norm{A \sqrt{P_Y}-\sqrt{P_Z}}_2^2 \\
\text{s.t.} \optspace & A^T \sqrt{P_Z} = \sqrt{P_Y}, \, A \geq 0
\end{aligned}
\end{equation}
where $\lambda > 0$ is a hyperparameter that controls how strictly the $A \sqrt{P_Y}=\sqrt{P_Z}$ constraint is imposed. In other words, the solution no longer has to induce clusters with exactly $P_Z$ as their marginal pmf, but it incurs a penalty proportional to the squared $\ell^2$-norm of the difference $A \sqrt{P_Y} - \sqrt{P_Z}$. Note that any other differentiable distance between distributions can be substituted here.  

The gradients of the components in the objective function of \eqref{Eq: Relaxed Optimization} are:
\begin{align}
\frac{\partial}{\partial A} \norm{A B}_{F}^2 &= \frac{\partial}{\partial A} \tr\!\left(ABB^T A^T\right) = 2 A B B^T \\
\frac{\partial}{\partial A} \norm{A v - w}_2^2 &= 2 \left( A v v^T - w v^T \right)
\end{align}
where $v = \sqrt{P_Y}$, $w = \sqrt{P_Z}$, and we use denominator layout notation (or Hessian formulation).

Furthermore, since there is an equivalence between \eqref{Eq: Frobenius Problem Alternative Form} and \eqref{Eq: Frobenius Problem DTM Form}, the remaining constraints in \eqref{Eq: Relaxed Optimization} correspond exactly to the second and third constraints in \eqref{Eq: Frobenius Problem Alternative Form} which are just enforcing $P_{Z|Y}$ to be a valid column stochastic matrix. Thus, we can either use any existing algorithms (e.g. \cite{Duchietal2008, ChenYe2011}) for projection back onto the simplex and apply them column-wise to $P_{Z|Y}$ or revise them to operate on $A$ directly. Algorithm \ref{Algorithm: SGD} describes the entire optimization procedure for problem \eqref{Eq: Relaxed Optimization}.

\begin{algorithm}
	\caption{Gradient Ascent Algorithm for Frobenius Norm Formulation}
	\label{Algorithm: SGD}
	\begin{flushleft}
		\textbf{Input:} Joint distribution $P_{Y,X}$, target marginal $P_Z$, marginal penalty multiplier $\lambda > 0$,  step size $\alpha > 0$ \\
		\textbf{Output:} Soft clusters induced by $P_{Z|Y}$
	\end{flushleft}
	
	\begin{algorithmic}[1] 
		
		\STATE Initialize $A_0 \in \R^{|\Z| \times |\Y|}$ to be an entry-wise positive matrix
		\STATE $B \gets [P_Y]^{-\frac{1}{2}}P_{Y,X}[P_X]^{-\frac{1}{2}}$
		\STATE $M_1 \gets BB^T$
		\STATE $M_2 \gets \lambda \sqrt{P_Y} \sqrt{P_Y}^T$
		\STATE $M_3 \gets \lambda \sqrt{P_Z} \sqrt{P_Y}^T$
		
		\WHILE{$A_t$ not converged}
		\STATE $A_t \gets A_{t-1}\!\left(I_{|\Y|} + \alpha (M_1-M_2)\right)+\alpha M_3$
		\IF{$A_t$ violates constraint above tolerance}
		\STATE $A_t \gets$ proj$(A_t)$ 
		\ENDIF
		\ENDWHILE
		\STATE \textbf{return} $P_{Z|Y} \gets [P_Z]^{\frac{1}{2}}A_t[P_Y]^{-\frac{1}{2}}$
	\end{algorithmic}
\end{algorithm}

\subsection{Nuclear Norm Relaxation}
\label{Nuclear Norm Relaxation}

Let us consider a modified problem where we approximate the Frobenius norm in \eqref{Eq: Frobenius Problem} using a nuclear norm. This yields the problem in \eqref{Eq: General Problem} specialized to the $p = 1$ case. We further relax this problem by completely disregarding the constraint on the marginal to obtain:
\begin{equation}
\label{Eq: Nuclear Norm}
\max_{P_{Z|Y} \in \Simp_{\Z|\Y}}{\norm{B_{Z,X}}_{*}} 
\end{equation}
which defines a ``maximal nuclear norm coupling'' representing a desirable clustering assignment. To derive some intuition about this problem, we recall a well-known result from the literature. For any fixed channel $P_{Z|X} \in \Simp_{\Z|\X}$, the second largest singular value $\sigma_2(B_{Z,X})$ of $B_{Z,X}$ is the \textit{Hirschfeld-Gebelein-R\'{e}nyi maximal correlation} between $Z$ and $X$, which is given by:
\begin{align}
\sigma_2(B_{Z,X}) & = \max_{\substack{f:\Z \rightarrow \R, \, g:\X \rightarrow \R \, : \\\E[f(Z)] = \E[g(X)] = 0\\\E\!\left[f(Z)^2\right] = \E\!\left[g(X)^2\right] = 1}}{\E[f(Z)g(X)]} \\
& = \max_{\substack{f \in \R^{|\Z|}, \, g \in \R^{|\X|}: \\f^T P_Z = g^T P_X = 0\\ f^T [P_Z] f = g^T [P_X] g = 1}}{f^T P_{Z, X} g}
\label{Eq: Maximal Correlation}
\end{align}
where the equality can be easily justified using the \textit{Courant-Fischer variational characterization} of singular values (cf. \cite{Renyi1959}, \cite[Definition 3, Proposition 2]{MakurZheng2018}, and the references therein). In particular, the optimal $f^{\star}$ and $g^{\star}$ can be obtained in terms of singular vectors of $B_{Z,X}$ corresponding to the singular value $\sigma_2(B_{Z,X})$, and they serve as useful \textit{features} that capture the maximal correlation between $Z$ and $X$ \cite{Makuretal2015, Qiu2017}. From this perspective, \eqref{Eq: Nuclear Norm} maximizes the statistical dependence between $Z$ and $X$ as measured by the sum of maximal correlations (or singular values) subject to the Markov constraint $X \rightarrow Y \rightarrow Z$ for the purposes of probabilistic clustering.

To derive an algorithm for \eqref{Eq: Nuclear Norm} that also uses SVD structure, we consider a generalization of \eqref{Eq: Maximal Correlation}. Using \textit{Ky Fan's extremum principle}, cf. \cite[Theorem 3.4.1]{HornJohnson1991}, we obtain the relation: 
\begin{equation}
\label{Eq: Ky Fan Extremum Principle}
\norm{B_{Z,X}}_{*} = \max_{\substack{F \in \R^{|\Z| \times r}, \, G \in \R^{|\X| \times r}: \\ F^T [P_Z] F = G^T [P_X] G = I_{r}}}{\tr\!\left(F^T P_{Z, X} G\right)} 
\end{equation}
where $r = \min(|\X|,|\Z|)$. The proof of \cite[Theorem 3.4.1]{HornJohnson1991} also shows that the optimal solutions of \eqref{Eq: Ky Fan Extremum Principle} are:
\begin{equation}
\label{Eq: ACE}
F^{\star} = [P_Z]^{-\frac{1}{2}} U \quad \text{and} \quad G^{\star} = [P_X]^{-\frac{1}{2}} V
\end{equation}
where $U \in \R^{|\Z| \times r}$ and $V \in \R^{|\X| \times r}$ are matrices with orthonormal columns that correspond to the left and right singular vector bases of the DTM $B_{Z,X}$, respectively. Thus, since $P_{Z,X} = P_{Z|Y} P_{Y,X}$ by the Markov property, we can rewrite \eqref{Eq: Nuclear Norm} as:
\begin{equation}
\label{Eq: Nuclear Norm Relaxation}
\max_{\substack{P_{Z|Y} \in \Simp_{\Z|\Y}, \\
F \in \R^{|\Z| \times r}, \, G \in \R^{|\X| \times r}: \\ F^T [P_Z] F = G^T [P_X] G = I_r}}{\tr\!\left(F^T P_{Z|Y} P_{Y,X} G\right)} .
\end{equation}
Inspired by \cite{Nieetal2017}, we also use \textit{alternating maximization} to solve this problem. With $P_{Z|Y}$ fixed, the optimal $F$ and $G$ are given by \eqref{Eq: ACE}. With $F$ and $G$ fixed, the objective function in \eqref{Eq: Nuclear Norm Relaxation} is linear in the entries of $P_{Z|Y}$ and can be solved using any \textit{linear programming} (LP) packages. Algorithm \ref{Algorithm: Alternating Maximization} describes the entire optimization procedure. 

\begin{algorithm}
	\caption{Alternating Maximization Algorithm for Nuclear Norm Formulation}
	\label{Algorithm: Alternating Maximization}
	\begin{flushleft}
		\textbf{Input:} Joint distribution $P_{Y,X}$ \\
		\textbf{Output:} Clusters induced by $P_{Z|Y}$
	\end{flushleft}
	
	\begin{algorithmic}[1] 
		
		\STATE Initialize $P_{Z|Y}$ to be a $|\Z| \times |\Y|$ column stochastic matrix
        \STATE $P_X \gets \1_{|\Y|}^T P_{Y,X}$
        \WHILE{$P_{Z|Y}$ not converged}
        \STATE $P_{Z,X} \gets P_{Z|Y} P_{Y,X}$
        \STATE $P_Z \gets P_{Z,X} \1_{|\X|}$
		\STATE $B \gets [P_Z]^{-\frac{1}{2}}P_{Z,X}[P_X]^{-\frac{1}{2}}$
		\STATE $U, \Sigma, V \gets \text{SVD}(B)$
        \STATE $F \gets [P_Z]^{-\frac{1}{2}} U$
		\STATE $G \gets [P_X]^{-\frac{1}{2}} V$
		\STATE $P_{Z|Y} \gets \argmax_{P_{Z|Y} \in \Simp_{\Z|\Y}} \tr\!\left(F^T P_{Z|Y} P_{Y,X} G\right)$
		\ENDWHILE
		\STATE \textbf{return} $P_{Z|Y}$
	\end{algorithmic}
\end{algorithm}

We remark that this algorithm does not require any prior knowledge of $P_Z$. This is one potential advantage of the relaxed nuclear norm formulation in \eqref{Eq: Nuclear Norm} over the original Frobenius norm formulation in \eqref{Eq: Frobenius Problem}. On the other hand, problem \eqref{Eq: Nuclear Norm Relaxation} has the uncommon feature that the constraint on $F$ depends on $P_{Z|Y}$ (or more precisely, on $P_Z$, which is derived from $P_{Z|Y}$). In typical instances of alternating maximization problems, the feasible sets of the variables (over which we alternate) are ``independent'' of each other (see e.g. \cite{CsiszarTusnady1984}). One way to ``decouple'' the feasible set of $F$ from $P_{Z|Y}$ is to fix some $P_Z$ (when we have prior knowledge). This imposes an additional linear constraint on $P_{Z|Y}$ which is easily handled by an LP. In our experiments, we do not impose this additional constraint because Algorithm \ref{Algorithm: Alternating Maximization} converges to a reasonable solution without the constraint.

\section{Experiments}
\label{Experiments}

\subsection{Word Embedding for MSR Sentence Completion Challenge}
\begin{table}[t]
	\caption{Performance comparison with other single architecture techniques as reported in~\cite{Mikolovetal2013}.}
	\label{Table: Word Embedding}
	\begin{center}
		\begin{small}
			\begin{sc}
				\begin{tabular}{cc}
					\toprule
					Architecture & Accuracy \\
					\midrule
					4-gram & $39\%$ \\
					Average LSA similarity & $49.6\%$ \\
					Log-bilinear model & $54.8\%$ \\
					RNNLMs & $55.4\%$ \\
					Skip-gram & $48.0\%$ \\
					\textbf{Our model} & $\mathbf{53.94\%}$ \\
					\bottomrule
				\end{tabular}
			\end{sc}
		\end{small}
	\end{center}
	\vskip -0.1in
\end{table}	

Though this paper is about clustering, we first want to validate that the DTM is an informative matrix for large scale unsupervised learning. To do this, we use it to learn word embeddings for the MSR Sentence Completion Challenge \cite{ZweigBurges2011}. The dataset consists of a training corpus of raw text taken from classic English literature and 1040 Scholastic Aptitude Test (SAT) style sentence completion questions.

\begin{table*}[t]
	\caption{Examples from the top $100$ most rated movies divided into the clusters found by Algorithm \ref{Algorithm: Alternating Maximization}. Note that cluster 2 is empty because it only contains movies outside the top 100 most rated movies.}
	\label{Table: Movie Clustering}
	\begin{center}
		\begin{small}
			\begin{sc}
				\begin{tabular}{ccccc}
					\toprule
                    Cluster 1 & Cluster 2 & Cluster 3 & Cluster 4 & Cluster 5 \\
					\midrule
                    Raiders of the Lost Ark & N/A & The Terminator & Star Wars & Contact \\
                    The Godfather & N/A & Terminator 2 & Return of the Jedi & Liar Liar \\
                    Pulp Fiction & N/A & Braveheart & Fargo & The English Patient \\
                    Silence of the Lambs & N/A & The Fugitive & Toy Story & Scream \\
					\bottomrule
				\end{tabular}
			\end{sc}
		\end{small}
	\end{center}
	\vskip -0.1in
\end{table*}

Let $P_{Y,X}$ be the normalized word-word co-occurrence matrix and let $U \Sigma V^T \approx [P_Y]^{-\frac{1}{2}} P_{Y,X} [P_X]^{-\frac{1}{2}}$ be the $640$-dimensional truncated SVD of the DTM. We use the \textit{alternating conditional expectations} (ACE) algorithm \cite{BreimanFriedman1985,Makuretal2015} to approximate $[P_Y]^{-\frac{1}{2}} U$, and use that as the word embedding.

We use various functions of cosine similarity between the candidate word and the surrounding words to select the most probable answer. Table \ref{Table: Word Embedding} shows that our method is competitive with popular single architecture word embedding techniques. This is not entirely surprising as there are other papers such as \cite{PenningtonSocherManning2014}, \cite{LevyGoldberg2014}, and \cite{StratosCollinsHsu2015} that advocate approximately factorizing various versions of the co-occurrence matrix. However, it provides empirical evidence that our method is valid and worth investigating more (on embedding as well as clustering).

\subsection{MovieLens 100K}

For qualitative validation, we use Algorithm \ref{Algorithm: Alternating Maximization} to find $5$ clusters using the MovieLens 100K dataset. The data is in the form of a movie-user rating matrix, where each entry can be blank to denote unrated, or in the range $\{1,\dots,5\}$. This is conceptually different from a co-occurrence matrix since a $5$-rated movie does not mean a user watched that movie $5$ times more frequently compared to a $1$-rated movie.

For preprocessing, we replace all blank entries with $0$ to denote no co-occurrence. We assume each unit increment in rating corresponds to tripling of a user's affinity toward a movie. Thus, we map each valid rating using the function $r \mapsto 3^{r-1}-1$. Then, we row normalize such that each row (corresponding to one movie) sums to $1$.

From Table \ref{Table: Movie Clustering}, we can see an approximate division of genres among clusters $1$, $3$, $4$, and $5$. Cluster $2$ captures many of the less popular movies and does not contain any one from the set of $100$ movies with the most ratings. Since MovieLens 100K does not contain ground truth cluster labels, we do not experiment further beyond this qualitative example.

\subsection{Reuters21578}
\begin{table}[t]
	\caption{Clustering accuracy on Reuters21578 for Algorithm \ref{Algorithm: Alternating Maximization}. The nuclear norm increases more slowly when $k \geq 8$, which implies that $k = 8$ or $10$ is the ``right'' number of clusters.}
	\label{Table: Reuters}
	\begin{center}
		\begin{small}
			\begin{sc}
				\begin{tabular}{ccccc}
					\toprule
					$k$ & Coverage & Overall acc. & $k$-acc. & $\norm{\cdot}_*$ \\
					\midrule
					$2$ & $69.55\%$ & $65.15\%$ & $93.67\%$ & $1.71$ \\
                    $3$ & $73.42\%$ & $65.51\%$ & $89.22\%$ & $2.33$ \\
                    $4$ & $77.01\%$ & $62.25\%$ & $80.83\%$ & $2.85$ \\
                    $6$ & $82.35\%$ & $57.43\%$ & $69.74\%$ & $3.72$ \\
                    $8$ & $85.43\%$ & $54.11\%$ & $63.34\%$ & $4.49$ \\
                    $10$ & $87.85\%$ & $48.52\%$ & $55.23\%$ & $5.14$ \\
					\bottomrule
				\end{tabular}
			\end{sc}
		\end{small}
	\end{center}
	\vskip -0.1in
\end{table}

\begin{table}[t]
	\caption{Clustering accuracy on Reuters21578 for Algorithm \ref{Algorithm: SGD}. Knowing the true cluster marginal helps maintain accuracy as $k$ increases.}
	\label{Table: Reuters SGD}
	\begin{center}
		\begin{small}
			\begin{sc}
				\begin{tabular}{ccccc}
					\toprule
					$k$ & Coverage & Overall acc. & $k$-acc. & $\norm{\cdot}_F$ \\
					\midrule
					$2$ & $69.55\%$ & $47.86\%$ & $68.81\%$ & $1.19$ \\
					$3$ & $73.42\%$ & $59.60\%$ & $81.18\%$ & $1.30$ \\
					$4$ & $77.01\%$ & $68.64\%$ & $89.12\%$ & $1.38$ \\
					$6$ & $82.35\%$ & $67.70\%$ & $82.21\%$ & $1.48$ \\
					$8$ & $85.43\%$ & $69.12\%$ & $80.90\%$ & $1.51$ \\
					$10$ & $87.85\%$ & $70.73\%$ & $80.52\%$ & $1.59$ \\
					\bottomrule
				\end{tabular}
			\end{sc}
		\end{small}
	\end{center}
	\vskip -0.1in
\end{table}

The Reuters21578 dataset contains $8293$ documents and their frequencies on $18933$ terms. Although the ground truth shows $65$ topic clusters, the largest $10$ clusters include $87.9\%$ of all documents while the smallest $8$ clusters each has $1$ document. Thus, we argue that a good algorithm needs to provide a metric to infer a meaningful number of clusters. 

For this experiment, we do not perform any data preprocessing and classify all documents into $k \in \{2,3,4,6,8,10\}$ clusters. Because we do not have clusters devoted to the $65-k$ smallest clusters, in Table \ref{Table: Reuters}, we report the classification accuracy in two ways. Overall accuracy counts all documents from those smallest clusters as incorrectly classified, and $k$-accuracy disregards those documents and only reports accuracy of documents from the top $k$ clusters. In both of these cases, the extra documents from the smallest clusters are still present in the data, acting as noise.

Similar to spectral clustering \cite{NgJordanWeiss2001}, we can plot the norm given by Algorithm \ref{Algorithm: Alternating Maximization} against $k$ to identify the $k$ that strikes a balance between document coverage and classification accuracy. At the cost of disregarding the smallest clusters, we achieve improved overall accuracy compared to the best algorithm ($43.94\%$) reported in \cite[Table 2]{Nieetal2017}.

Alternatively, assuming we have access to the ground truth cluster marginal pmf, we can use Algorithm \ref{Algorithm: SGD}. Table \ref{Table: Reuters SGD} shows that this prior information offers significant improvements in accuracy as $k$ gets large.

\section{Conclusion and Future Work}
\label{Conclusion}

In this paper, we reviewed the mutual information formulation for probabilistic clustering \eqref{Eq: Information Theoretic Formulation}. Then, to convert \eqref{Eq: Information Theoretic Formulation} into a matrix optimization \eqref{Eq: Frobenius Problem}, we locally approximated mutual information as the Frobenius norm of the DTM in Proposition \ref{Prop: Local Mutual Information}. This allowed us to explicitly learn a maximal matrix norm coupling $P_{Z|Y}$ for clustering as opposed to the standard procedure (embedding and $k$-means). Learning $P_{Z|Y}$ also lets us encode prior information. We saw one example of this with the predefined $P_Z$ in \eqref{Eq: Frobenius Problem}. We can also add constraints that fix certain columns of $P_{Z|Y}$ if a subset of the data is labeled to perform semi-supervised learning.

There are two aspects of our approach that can be improved in future. Firstly, we can implement more efficient non-convex optimization algorithms that converge to solutions closer to the global optimum. Secondly, we can improve our model's robustness to noise. Currently, we treat the observed noisy co-occurrence matrix as a good estimate of the true distribution while \textit{matrix factorization} (MF) approaches treat the noise as entry-wise Gaussian perturbations of a low rank model \cite{SalakhutdinovMnih2007}. In our experience, MF tends to perform well on data with high entry-wise noise while our approach performs well on data with complex community structures and lower noise.

Another future direction is to probabilistically cluster $X$ in addition to $Y$. The optimization problem for this is:
\begin{equation}
\begin{aligned}
\max_{\substack{A \in \R^{|\Z| \times |\Y|} , \\ C \in \R^{|\W| \times |\X|}}} \optspace & \norm{A B C^T}_{F}^2 \\
\text{s.t.} \optspace & A \sqrt{P_Y} = \sqrt{P_Z} , \, A^T \sqrt{P_Z} = \sqrt{P_Y}, \\
& C \sqrt{P_X} = \sqrt{P_W}, \, C^T \sqrt{P_W} = \sqrt{P_X}, \\
& A \geq 0, \, C \geq 0 .
\end{aligned}
\end{equation}
where $C$ obtains the clusters of $X$, cf. \eqref{Eq: Frobenius Problem DTM Form}. This parallels the notion of \textit{co-clustering} in the literature \cite{DhillonMallelaModha2003}, and is a topic worthy of further investigation.

\balance
\bibliographystyle{IEEEtran}
\bibliography{ClusteringRefsFinal}

\end{document}